\documentclass[a4paper,11pt]{article}

\usepackage[margin=1.2in]{geometry}

\usepackage{float}
\usepackage{subcaption}
\usepackage{amsmath}
\usepackage{algorithmic}
\usepackage{algorithm}
\usepackage{amsthm}
\usepackage{amsfonts}
\usepackage{comment}
\usepackage{graphicx}
\usepackage{hyperref}
\usepackage{color}
\usepackage{multirow}

\newtheorem{theorem} {Theorem}
\newtheorem{lemma} {Lemma}

\newtheorem{corollary} {Corollary}
\newtheorem{assumption} {Assumption}
\newtheorem{observation} {Observation}

\newtheorem{remark} {Remark}

\def\u{{\mathbf{u}}}
\def\v{{\mathbf{v}}}

\def\X{{\mathbf{X}}}
\def\Y{{\mathbf{Y}}}
\def\A{{\mathbf{A}}}
\def\M{{\mathbf{M}}}

\def\B{{\mathbf{B}}}

\def\V{{\mathbf{V}}}

\def\var{{\mathbf{Var}}}

\newcommand{\mX}{\mathcal{X}}

\newcommand{\mF}{\mathcal{F}}

\newcommand{\mS}{\mathcal{S}}
\newcommand{\mbS}{\mathbb{S}}
\newcommand{\mD}{\mathcal{D}}
\newcommand{\E}{\mathbb{E}}

\newcommand{\nnz}{\textrm{nnz}}
\newcommand{\trace}{\textrm{Tr}}
\newcommand{\rank}{\textrm{rank}}
\newcommand{\reals}{\mathbb{R}}

\title{On the Convergence of Stochastic Gradient Descent with Low-Rank Projections for Convex Low-Rank Matrix Problems}
\author{Dan Garber \\
Technion - Israel Institute of Technology \\
{\small{dangar@technion.ac.il}}}
\date{}

\begin{document}
 \maketitle

\begin{abstract}
We revisit the use of Stochastic Gradient Descent (SGD) for solving convex optimization problems that serve as highly popular convex relaxations for many important low-rank matrix recovery problems such as \textit{matrix completion}, \textit{phase retrieval}, and more. The computational limitation of applying SGD to solving these relaxations in large-scale is the need to compute a potentially high-rank singular value decomposition (SVD) on each iteration in order to enforce the low-rank-promoting constraint. We begin by considering a simple and natural sufficient condition so that these relaxations indeed admit low-rank solutions. This condition is also necessary for a certain notion of low-rank-robustness to hold. Our main result shows that under this condition which involves the eigenvalues of the gradient vector at optimal points, SGD with mini-batches, when initialized with a ``warm-start" point, produces iterates that are low-rank with high probability, and hence only a low-rank SVD computation is required on each iteration. This suggests that SGD may indeed be practically applicable to solving large-scale convex relaxations of low-rank matrix recovery problems. Our theoretical results are accompanied with supporting preliminary empirical evidence. As a side benefit, our analysis is quite simple and short.
\end{abstract}

\section{Introduction}
This paper is concerned with convex optimization formulations and algorithms for low-rank matrix recovery. Low-rank matrix recovery problems have numerous applications in machine learning, statistics and related field and have received much attention in recent years, with some of the most well known problems / applications being \textit{matrix completion} \cite{Candes09, Recht11, Jaggi10,ge2016matrix}, \textit{phase retrieval} \cite{candes2015phase, netrapalli13, Yurtsever17}, \textit{robust PCA} \cite{Candes11, wright2009robust, netrapalli2014non, yi2016fast, mu2016scalable}, and more. However, these optimization problems are often NP-Hard to solve due to the explicit low-rank constraint / objective. To circumvent this difficulty, a significant body of work in recent years has been devoted to study convex relaxations to these problems, which are computationally tractable, and also often well motivated in terms of their ability to recover the correct low-rank solution (usually under certain statistical assumptions), see for instance \cite{Candes09, Recht11, candes2015phase, Candes11,wright2009robust}. These convex relaxations replace the explicit non-convex low-rank constraint / objective with a convex surrogate such as the sum of the singular values of the matrix, often called the nuclear norm, or the trace norm. Importantly, these convex relaxations can be formulated in the following canonical form (see for instance \cite{Jaggi10}), which is also the main optimization problem under consideration in this paper:
\begin{eqnarray}\label{eq:optProb}
\min_{\X\in\mS_n}f(\X).
\end{eqnarray}
Here $\mS_n$ denotes the spectrahedron in $\mbS^n$ (space of $n\times n$ real symmetric matrices), i.e., $\mS_n := \{\X\in\mbS^n ~|~ \X\succeq 0,~\trace(\X)=1\}$. Throughout this work, $f$ is assumed $\beta$-smooth (Lipschitz gradient) and convex. 

Additionally, motivated by cases in which $f(\X)$ admits a finite-sum structure, i.e., $f(\X) := \frac{1}{m}\sum_{i=1}^mf_i(\X)$, where the number of functions $m$ is large and hence the computation of exact gradients of $f(\cdot)$ is prohibitive, or when $f(\X)$ is given by an expectation w.r.t. some unknown distribution, i.e., $f(\X) :=\E_{g\sim\mD}[g(\X)]$, and only a finite sample drawn i.i.d. form $\mD$ is available (e.g., in statistically-motivated scenarios), we consider stochastic optimization methods for solving Problem \eqref{eq:optProb}. Concretely, we assume the standard generic model for first-order stochastic optimization, in which $f(\cdot)$ is given by a stochastic first-order oracle, which when queried with some point $\X\in\mbS^n$ returns a random matrix $\widehat{\nabla}\in\mbS^n$ satisfying the following standard assumptions:
\begin{align*}
i.~\E[\widehat{\nabla}~|~\X] = \nabla{}f(\X), ~ ii.~\Vert{\widehat{\nabla}}\Vert_F \leq G, ~\Vert{\widehat{\nabla}}\Vert \leq B, ~ iii.~ \E[\Vert{\widehat{\nabla}-\nabla{}f(\X)}\Vert_F^2~|~\X] \leq \sigma^2,
\end{align*}
%\begin{align*}
%i.~\E[\widehat{\nabla}~|~\X] = \nabla{}f(\X), \quad ii.~\Vert{\widehat{\nabla}}\Vert_F \leq G, ~\Vert{\widehat{\nabla}}\Vert \leq B, \quad iii.~ \var(\widehat{\nabla}~|~\X) \leq \sigma^2,
%\end{align*}
for some $G,B,\sigma^2 >0$, where for any matrix $\M\in\mbS^n$. $\Vert{\M}\Vert_F$ denotes the Frobenius (Euclidean) norm, and $\Vert{\M}\Vert$ denotes the spectral norm (largest singular value).

While Problem \eqref{eq:optProb} is convex, it is still highly challenging to solve in large-scale via traditional first-order methods, such as projected gradient methods \cite{Nesterov13, Bubeck15, HazanK14, Rakhlin12} or conditional gradient-based methods \cite{Jaggi13b, Hazan12, Lan16, Hazan16, Garber18}, since these require a potentially high-rank singular value decomposition (SVD) computation on each iteration (which can take as much as $O(n^3)$ runtime), and / or to store potentially high-rank matrices in memory (despite the often implicit assumption that the optimal solution is low-rank).

As a starting point let us recall the structure of the Euclidean projection onto the spectrahedron $\mS_n$, which we denote as $\Pi_{\mS_n}[\cdot]$.

\begin{lemma}[Projection onto the spectrahedron]\label{lem:spectrahedronProj}
Let $\M\in\mbS^n$ and write its eigen-decomposition as $\M = \sum_{i=1}^n\lambda_i\v_i\v_i^{\top}$. Then, it holds that
%\begin{eqnarray*}
$\Pi_{\mS_n}[\M] = \sum_{i=1}^n\max\{0,~\lambda_i-\lambda\}\v_i\v_i^{\top}$,
%\end{eqnarray*}
where $\lambda\in\reals$ is the unique scalar satisfying $\sum_{i=1}^n\max\{0,~\lambda_i-\lambda\} =1$.
\end{lemma}

From the lemma it is quite obvious why at worst-case computing this projection may require a high-rank SVD (note that given the SVD of $\M$, computing the threshold parameter $\lambda$ could be done in $O(n\log{}n)$ time via sorting). From this lemma we also make the following simple yet important observation.
\begin{observation}[Low-rank projection requires low-rank SVD]\label{observ:lowRank}
Given a matrix $\M\in\mbS^n$, if $\rank\left({\Pi_{\mS_n}[\M]}\right) =r$, then only the top-$r$ components in the SVD of $\M$ (corresponding to the rank-$r$ matrix $\sum_{i=1}^r\lambda_i\v_i\v_i^{\top}$) are required to compute the projection. Hence, only a rank-$r$ SVD of $\M$ is required. \footnote{In particular, according to Lemma \ref{lem:spectrahedronProj}, if $\M$ admits the eigen-decomposition $\sum_{i=1}^n\lambda_i\v_i\v_i^{\top}$, then its projection onto $\mS_n$ is rank-$r$ if and only if $\sum_{i=1}^r\lambda_i \geq 1+ r\cdot\lambda_{r+1}$.}
\end{observation}

This observation implies that when the projected matrix is low-rank, the projection can be computed via fast iterative methods (such as power iterations or the faster Lanczos method) with runtime that is proportional to only $r\cdot\nnz(\M)$ (where $\nnz(\cdot)$ denotes the number of non-zero entries), as opposed to $n^3$ required for a full-rank SVD.

Let us denote by $\mX^*$ the set of optimal solutions to Problem \eqref{eq:optProb}.  Our main result in this paper is that given some optimal solution $\X^*\in\mX^*$ with $\rank(\X^*)=r$, under a simple and natural condition on the eigenvalues of the gradient vector $\nabla{}f(\X^*)$, which we present next, the standard projected stochastic gradient method with mini-batches (see Algorithm \ref{alg:sgd}), when initialized close enough to $\X^*$, will converge with constant probability to the optimal value of Problem \eqref{eq:optProb} - $f^*$, while requiring on each iteration a single SVD computation of rank at most $r$ to compute the projection. 

\begin{assumption}\label{ass:gap}
We say an optimal solution $\X^*\in\mX^*$ of rank $r$ satisfies the eigen-gap assumption if $\lambda_{n-r}(\nabla{}f(\X^*))-\lambda_n(\nabla{}f(\X^*))  > 0$.
\end{assumption}

Importantly, the eigen-gap assumption, even without assuming explicitly that $\X^*$ is of rank $r$,  is a sufficient condition for $\X^*$ to have rank at most $r$. This follows from the following lemma (see Lemma 7 in \cite{Garber19}). Thus,   the additional requirement that $\X^*$ is of rank exactly $r$ could be understood as a non-degeneracy requirement.

\begin{lemma}\label{lem:eigsOfOptGradSDP}
Let $\X^*\in\mX^*$ be any optimal solution and write its eigendecomposition as $\X^* = \sum_{i=1}^r\lambda_i\v_i\v_i^{\top}$. Then, the gradient vector $\nabla{}f(\X^*)$ admits an eigendecomposition such that the set of vectors $\{\v_i\}_{i=1}^r$ is a set of top eigen-vectors of $(-\nabla{}f(\X^*))$ which corresponds to the eigenvalue $\lambda_1(-\nabla{}f(\X^*)) = -\lambda_n(\nabla{}f(\X^*))$.
\end{lemma}

%While Assumption \ref{ass:gap} is a sufficient condition for the the existence of a rank-$r$ optimal solution, it is not a necessary condition. However, the following 
In order to better motivate Assumption \ref{ass:gap} we bring the following lemma which suggests that this condition is required for the robustness of low-rank optimal solutions. The lemma shows that when the eigengap assumption does not hold, performing a standard projected gradient step from this optimal point w.r.t. to an arbitrarily small perturbation of the optimization problem, will result in a higher-rank matrix. Here we recall the first-order optimality condition $\Pi_{\mS_n}[\X^*-\beta^{-1}\nabla{}f(\X^*)] = \X^*$.

The lemma is a simple adaptation of Lemma 3 in \cite{Garber19} (which considers optimization over trace-norm balls). A proof is given in the appendix for completeness.

\begin{lemma}\label{lem:robustRank}
Let $f:\mbS^n\rightarrow\reals$ be $\beta$-smooth and convex. Let $\X^*\in\mS_n$ be an optimal solution of rank $r$ to the optimization problem $\min_{\X\in\mS_n}f(\X)$. Let $\mu_1,\dots,\mu_n$ denote the eigenvalues of $\nabla{}f(\X^*)$ in non-increasing order. Then, $\mu_{n-r}=\mu_{n}$ if and only if for any arbitrarily small $\zeta>0$ it holds that
\begin{eqnarray*}
\rank(\Pi_{(1+\zeta)\mS_n}[\X^*-\beta^{-1}\nabla{}f(\X^*)]) > r,
\end{eqnarray*}
where $(1+\zeta)\mS_n = \{(1+\zeta)\X~|~\X\in\mS_n\}$, and $\Pi_{(1+\zeta)\mS_n}[\cdot]$ denotes the Euclidean projection onto the convex set $(1+\zeta)\mS_n$.
\end{lemma}

We also refer the reader to \cite{Garber19} (Table 2) for an empirical evidence that Assumption \ref{ass:gap} seems to be quite practical for real-world datasets.

%Importantly, while Assumption \ref{ass:gap} is concerned with the gradient vector at a specific optimal question, many popular objective functions take the form $f(\X) = g(\mA\X)$, where $g(\cdot)$ is strongly convex and $\mA$ is a linear mapping (for instance matrix completion and phase retrieval take this form with $g$ being the squared Euclidean norm). For such functions it holds that the gradient vector is constant over the optimal set \footnote{since $\nabla{}f(\X) = \mA^{\top}\nabla{}g(\mA\X)$ and due to the strong convexity of $g$, $\mA\X$ must be constant over the optimal set $\X^*$}. Thus, for this highly popular structure, if Assumption \ref{ass:gap} holds for some optimal solution, then the eigen-gap holds for all optimal solutions.

\begin{algorithm}
\caption{Projected Stochastic Gradient Descent with minibathces}
\label{alg:sgd}
\begin{algorithmic}[1]
\STATE input: initialization point $\X_1\in\mS_n$, batch-size $L$, time horizon $T$, sequence of step-sizes $\{\eta_t\}_{t\in[T-1]}$
\FOR{$t=1\dots{}T-1$}
\STATE $\widehat{\nabla}_t \gets \frac{1}{L}\sum_{i=1}^L\widehat{\nabla}_t^{(i)}$, where $\{\widehat{\nabla}_t^{(i)}\}_{i=1}^L$ is produced by $L$ calls to the stochastic oracle of $f(\cdot)$ with the input point $\X_t$
\STATE $\X_{t+1} \gets \Pi_{\mS_n}[\X_t - \eta_t\widehat{\nabla}_t]$
\ENDFOR
\STATE return solution $\bar{\X}$ according to one of the following options:
%\vspace{-15pt}
\begin{eqnarray*}
\bar{\X} \gets \X_{t_0}, ~ t_0\sim{}Uni\{1,\dots,T\} ~ \textrm{(\textbf{option I})} \quad \textrm{or} \quad
\bar{\X} \gets \frac{1}{T}\sum_{i=1}^T\X_i ~ \textrm{(\textbf{option II})} 
\end{eqnarray*}
%\begin{eqnarray*}
%&\bar{\X} \gets \X_{t_0}, ~ t_0\sim{}Uni\{1,\dots,T\} \quad \textrm{(\textbf{option I})} &\\
%&\textrm{or}& \\
%&\bar{\X} \gets \frac{1}{T}\sum_{i=1}^T\X_i \quad \textrm{(\textbf{option II})} &
%\end{eqnarray*}

\end{algorithmic}
\end{algorithm}

Formally, the main result of this paper is the proof of the following theorem.

\begin{theorem}\label{thm:main}
Let $\X^*\in\mX^*$ be an optimal solution of rank $r$ which satisfies Assumption \ref{ass:gap}. Consider running SGD  (Algorithm \ref{alg:sgd}) for $T$ iterations with a fixed step-size $\eta = \frac{R_0}{10G\sqrt{T}\log(8T)}$ and when the first iterate $\X_1$ satisfies $\rank(\X_1) \leq r$ and $\Vert{\X_1-\X^*}\Vert_F \leq R_0/2$, where 
\begin{align*}
R_0 := \frac{1}{8}\Big({r\beta + \frac{\sqrt{2r}B}{\lambda_r(\X^*)}}\Big)^{-1}\delta,\qquad 0 <\delta \leq \lambda_{n-r}(\nabla{}f(\X^*)) - \lambda_n(\nabla{}f(\X^*)),
\end{align*}
and with constant minibatch size $L$ satisfying $L\geq L_0 =O\Big({\max\{(\left({\frac{\sigma}{GR_0}}\right)^2,\frac{B^2r^2}{\delta^2}\log(nT)\}}\Big)$. Then, for any $T$ sufficiently large, it holds with probability at least $1/2$ that
\begin{enumerate}
\item
$f(\bar{\X}) - f^* = O\left({\frac{GR_0\log{T}}{\sqrt{T}}}\right)$,
\item
$\forall t\in[T-1]$: $\rank(\X_{t+1}) \leq r$. Moreover, if \textbf{option I} is used for the returned solution, then $\rank(\bar{\X}) \leq r$.
\end{enumerate}

%\begin{eqnarray*}
%&i)& \qquad f(\bar{\X}) - f^* = O\left({\frac{GR_0\log{T}}{\sqrt{T}}}\right), \\
%&ii)& \qquad \forall t\in[T-1]:\quad \rank(\X_{t+1}) \leq r,\\
%&iii)& \qquad \textrm{using \textbf{option I} it holds that } \rank(\bar{\X}) \leq r.
%\end{eqnarray*}
\end{theorem}

Thus, Theorem \ref{thm:main}, together with Observation \ref{observ:lowRank}, imply that with constant probability, all the steps of SGD can be computed via a rank-$r$ SVD.

\begin{corollary}[sample complexity]\label{corr:samplecomplexity}
The overall sample complexity to achieve $f(\bar{\X})-f^*\leq \epsilon$ with probability at least $1/2$, when initializing from a ``warm-start'', is upper-bounded by $\tilde{O}\left({\epsilon^{-2}\max\{\sigma^2,\lambda_r^2(\X^*)rG^2\}}\right)$\footnote{Throughout this paper we use the notation $\tilde{O}(\cdot)$ or $\tilde{\Theta}(\cdot)$ to suppress poly-logarithmic factors.} (note that $\lambda_r(\X^*) \leq 1/r$).
\end{corollary}

The proof is given in the appendix. We note this sample complexity is nearly optimal (up to a logarithmic factor) in $\epsilon$ and optimal in $\sigma,G$ (see for instance \cite{Bubeck15}). Most importantly, it is independent of the eigen-gap $\delta$.\footnote{Naturally, the sample complexity to obtain the required ``warm-start" initialization will depend on $\delta$, but will be independent of the overall target accuracy $\epsilon$.}

\begin{remark}
It is quite important to note that while verifying the validity of Assumption \ref{ass:gap}, or the ''warm start" condition, or even setting the step-size in Theorem \ref{thm:main} correctly, can be quite difficult in practice, from a practical point of view, it is mainly important that the low-rank-SVD-based projection is indeed the correct Euclidean projection. This however, could be easily verified in each step $t$ of the algorithm: if instead of computing a rank-$r$ SVD of the point to project $\X_t-\eta_t\widehat{\nabla}_t$, we compute a rank-$(r+1)$ SVD, we can easily verify (using the condition on the thresholding parameter $\lambda$ in Lemma \ref{lem:spectrahedronProj}, see Footnote 1), if the correct projection is indeed of rank at most $r$, and hence verify that the algorithm indeed converges correctly.
\end{remark}

\subsection{Related work}
Our work is primarily motivated by the very recent work \cite{Garber19}, which considered Problem \eqref{eq:optProb} in a purely deterministic setting, i.e., when exact gradients of $f(\cdot)$ are available. In that work it is shown that, under Assumption \ref{ass:gap}, standard projected gradient methods, when initialized with a ``warm-start" point, converge with their original convergence guarantees to an optimal solution using only low-rank SVD to compute the projection. However, these results are not directly extendable to the stochastic setting for two reasons. First, the ``warm-start" requirement in \cite{Garber19} requires that the distance to an optimal solution is proportional to the step-size used. While this makes sense in the deterministic setting, since the typical step-size for projected-gradient methods is just $1/\beta$, for SGD, the step-size (e.g., when chosen to be fixed) is proportional to the target accuracy $\epsilon$, which imposes an unrealistic initialization requirement (in particular, given that the function is Lipschitz, such a condition already implies that the initial point satisfies $f(\X_1)-f^* = O(\epsilon)$). Therefore, our main technical contribution is to provide an alternative analysis to the one used in \cite{Garber19}, in which the required initial distance to an optimal solution is independent of the step-size.

Second, since the analysis of \cite{Garber19} (as the one in this work) only applies in a certain ball around an optimal solution, it relies on the property that the projected gradient method does not increase the distance to the optimal set from one iteration to the next. This property does not hold anymore for SGD, and here we introduce a martingale argument to show that with high probability all the iterates indeed stay within the relevant ball.

For specific low-rank matrix recovery problems, the works \cite{de2015global, Jin16, ge2016matrix, bhojanapalli2016global}  yield global convergence guarantees for non-convex SGD which forces the low-rank constraint by explicitly factorizing the matrix variable as the product of two rank-$r$ matrices. However, these only hold under very specific and quite strong statistical assumptions on the data. On the contrary, in this work we do not impose any statistical generative model on the data.

Finally, we note that works that analyze non-convex methods without relying on strong statistical models, such as \cite{Dropping16} (though they only consider the \textit{deterministic} gradient descent method), also require ``warm-start'' initialization which is qualitatively similar to ours, e.g., relies on the ratio between smallest and largest singular values of the optimal solution (see Theorem \ref{thm:main} above).

\section{Analysis}
The proof of Theorem \ref{thm:main} follows from combining the standard convergence analysis of SGD with two main lemmas. Lemma \ref{lem:goodProj}, which is the main technical novelty we introduce in this paper, and believe may be of independent interest, establishes (informally) that at any step $t$ of Algorithm \ref{alg:sgd}, if $\X_t$ is sufficiently close to an optimal solution $\X^*$ which satisfies the gap assumption (Assumption \ref{ass:gap}), and the stochastic gradient is not too noisy, then $\X_{t+1}$ is low-rank (and hence can be computed, given $\X_t,\widehat{\nabla}_t$, using only a low-rank SVD). Lemma \ref{lem:martingaleDistNew} then uses a martingale concentration argument to establish that, if $\X_1$ is sufficiently close to some optimal solution $\X^*$, then with high probability, all following iterates are also sufficiently close. Combining these two lemmas ensures that with high probability, the projection onto $\mS_n$ at each step of Algorithm \ref{alg:sgd} can be computed using only a low-rank SVD computation.

Throughout this work we let the operation $\A\bullet\B$ denote the standard inner product for any two matrices $\A,\B\in\reals^{m\times n}$, i.e., $\A\bullet\B = \trace(\A\B^{\top})$.

\begin{lemma}\label{lem:goodProj}
Let $\X^*\in\mX^*$ be of rank $r$,  and let $\mu_1\dots\mu_n$ denote the eigenvalues of $\nabla{}f(\X^*)$ in non-increasing order. Let $\X\in\mS_n$ be a matrix such that $\rank(\X) \leq r$, and suppose that
\begin{eqnarray}\label{eq:convRadius}
\Vert{\X-\X^*}\Vert_F \leq \frac{1}{4}\Big({r\beta + \frac{\sqrt{2r}{B}}{\lambda_r(\X^*)}}\Big)^{-1}\left({\delta - 4r\xi}\right), 
\end{eqnarray}
where $\delta := \mu_{n-r}-\mu_n$. Finally, let $\widetilde{\nabla}$ be a matrix such that $\Vert{\widetilde{\nabla}-\nabla{}f(\X)}\Vert \leq \xi$, $\Vert{\widetilde{\nabla}}\Vert\leq B$. Then, for any step-size $\eta >0$ it holds that $\rank\left({\Pi_{\mS_n}[\X-\eta\widetilde{\nabla}]}\right) \leq r$.
\end{lemma}

\begin{proof}
Let us denote $\Y = \X-\eta\widetilde{\nabla}$. From Lemma \ref{lem:spectrahedronProj} it follows that a sufficient condition so that $\rank\left({\Pi_{\mS_n}[\X-\eta\widetilde{\nabla}]}\right) \leq r$, is 
%\begin{eqnarray*}
$\sum_{i=1}^{r}\lambda_i(\Y) \geq 1 + r\cdot\lambda_{r+1}(\Y)$
%\end{eqnarray*}
(since then the thresholding parameter $\lambda$ in Lemma \ref{lem:spectrahedronProj} must satisfy $\lambda \geq \lambda_{r+1}(\Y)$).

Let $\X = \V\Lambda\V^{\top}$ denote the eigen-decomposition of $\X$. In case $\rank(\X) < r$, we extend this decomposition to have rank=$r$ by adding additional zero eigenvalues and corresponding eigenvectors, so $\V\in\reals^{d\times r}$. It holds that
{\small
\begin{align*}
\sum_{i=1}^{r}\lambda_i(\Y) \geq \V\V^{\top}\bullet\Y = \V\V^{\top}\bullet(\X-\eta\widetilde{\nabla})  =\trace(\X) - \eta\V\V^{\top}\bullet\widetilde{\nabla} = 1 - \eta\V\V^{\top}\bullet\widetilde{\nabla}.
\end{align*}}

Using the above inequality we also have
{\small
\begin{align*}
\lambda_{r+1}(\Y) &= \sum_{i=1}^{r+1}\lambda_i(\Y) - \sum_{j=1}^r\lambda_j(\Y) \leq \sum_{i=1}^{r+1}\lambda_i(\Y) - \left({1 - \eta\V\V^{\top}\bullet\widetilde{\nabla}}\right) \\
&\underset{(a)}{\leq}  \sum_{i=1}^{r+1}\left({\lambda_i(\X) + \lambda_i(-\eta\widetilde{\nabla})}\right) - \left({1 - \eta\V\V^{\top}\bullet\widetilde{\nabla}}\right) \\
&\underset{(b)}{=} 1 + \sum_{i=1}^{r+1}\lambda_i(-\eta\widetilde{\nabla}) - \left({1 - \eta\V\V^{\top}\bullet\widetilde{\nabla}}\right) = \eta\V\V^{\top}\bullet\widetilde{\nabla} -\eta\sum_{i=1}^{r+1}\lambda_{n-i+1}(\widetilde{\nabla}),
\end{align*}}
where (a) follows from Ky Fan's eigenvalue inequality, and (b) follows since $\sum_{i=1}^{r+1}\lambda_i(\X) =\sum_{i=1}^{\rank(\X)}\lambda_i(\X)= \trace(\X)=1$.
 Thus, we arrive at the following sufficient condition so that $\rank\left({\Pi_{\mS_n}[\X-\eta\widetilde{\nabla}]}\right) \leq r$: 
{\small
\begin{eqnarray*}
 1 - \eta\V\V^{\top}\bullet\widetilde{\nabla} \geq 1+ r\left({\eta\V\V^{\top}\bullet\widetilde{\nabla} -\eta\sum_{i=1}^{r+1}\lambda_{n-i+1}(\widetilde{\nabla})}\right),
\end{eqnarray*}}

which boils down to the sufficient condition
{\small
\begin{eqnarray}\label{eq:lem:proj:1}
-\V\V^{\top}\bullet\widetilde{\nabla} \geq -\frac{r}{r+1}\sum_{i=1}^{r+1}\lambda_{n-i+1}(\widetilde{\nabla}).
\end{eqnarray}}

Let $\V^*\Lambda^*\V^{*\top}$ denote the eigen-decomposition of $\X^*$ and recall $\rank(\X^*) = r$. Then,
{\small
\begin{align*}
-\V\V^{\top}\bullet\widetilde{\nabla} &\geq -\V^*\V^{*\top}\bullet\widetilde{\nabla} - \Vert{\V\V^{\top}-\V^*\V^{*\top}}\Vert_*\Vert{\widetilde{\nabla}}\Vert \\
&\geq -\V^*\V^{*\top}\bullet\nabla{}f(\X^*) - \Vert{\V^*\V^{*\top}}\Vert_*\cdot\Vert{\widetilde{\nabla}-\nabla{}f(\X^*)}\Vert\\
&- \Vert{\V\V^{\top}-\V^*\V^{*\top}}\Vert_*\cdot\Vert{\widetilde{\nabla}}\Vert\\
&\geq -\V^*\V^{*\top}\bullet\nabla{}f(\X^*) - r\Vert{\widetilde{\nabla}-\nabla{}f(\X^*)}\Vert - \sqrt{2r}\Vert{\widetilde{\nabla}}\Vert\cdot\Vert{\V\V^{\top}-\V^*\V^{*\top}}\Vert_F.
\end{align*}}

Since $\X^*$ is an optimal solution, it follows from Lemma \ref{lem:eigsOfOptGradSDP} that
%\vspace{-5pt}
{\small
\begin{align*}
-\V^*\V^{*\top}\bullet\nabla{}f(\X^*) = -r\mu_n.
\end{align*}}
Also, since $f(\cdot)$ is $\beta$-smooth, 
{\small
\begin{align}\label{eq:lem:proj:2}
\Vert{\widetilde{\nabla}-\nabla{}f(\X^*)}\Vert \leq \Vert{\widetilde{\nabla}-\nabla{}f(\X)}\Vert + \Vert{\nabla{}f(\X)-\nabla{}f(\X^*)}\Vert_F \leq \xi + \beta\Vert{\X-\X^*}\Vert_F.
\end{align}}

Finally, using the Davis-Kahan $\sin\theta$ theorem (see for instance Theorem 2 in \cite{Yu2014useful}), we have that
{\small
\begin{align*}
\Vert{\V\V^{\top}-\V^*\V^{*\top}}\Vert_F &\leq \frac{2\Vert{\X-\X^*}\Vert_F}{\lambda_r(\X^*)-\lambda_{r+1}(\X^*)}  = \frac{2\Vert{\X-\X^*}\Vert_F}{\lambda_r(\X^*)}.
\end{align*}}

Thus, combining these three bounds, we have that
\begin{align}\label{eq:lem:proj:4}
-\V\V^{\top}\bullet\widetilde{\nabla} &\geq  -r\mu_n -r\xi - \sqrt{r}\Vert{\X-\X^*}\Vert_F\left({\sqrt{r}\beta+\frac{2\sqrt{2}\Vert{\widetilde{\nabla}}\Vert}{\lambda_r(\X^*)}}\right).
\end{align}

On the other-hand,
\begin{align}\label{eq:lem:proj:5}
-\sum_{i=1}^{r+1}\lambda_{n-i+1}(\widetilde{\nabla}) &= \sum_{i=1}^{r+1}\lambda_{i}(-\widetilde{\nabla}) = \sum_{i=1}^{r+1}\lambda_i(-\nabla{}f(\X^*) + (\nabla{}f(\X^*)-\widetilde{\nabla}))\\
& \underset{(a)}{\leq} \nonumber \sum_{i=1}^{r+1}\lambda_i(-\nabla{}f(\X^*)) + \sum_{i=1}^{r+1}\lambda_i(\nabla{}f(\X^*)-\widetilde{\nabla}) \nonumber \\
&\leq  -\sum_{i=1}^{r+1}\lambda_{n-i+1}(\nabla{}f(\X^*)) + (r+1)\Vert{\nabla{}f(\X^*)-\widetilde{\nabla})}\Vert  \nonumber \\
&\underset{(b)}{=} -((r+1)\mu_n+\delta) + (r+1)\Vert{\nabla{}f(\X^*)-\widetilde{\nabla})}\Vert  \nonumber \\
&\underset{(c)}{\leq}  -((r+1)\mu_n+\delta) + (r+1)(\xi + \beta\Vert{\X-\X^*}\Vert_F), 
\end{align}
where (a) follows from Ky Fan's eigenvalue inequality, and (b) follows from Lemma \ref{lem:eigsOfOptGradSDP} and our assumption on the eigenvalues of $\nabla{}f(\X^*)$, and (c) follows from \eqref{eq:lem:proj:2}.

Plugging \eqref{eq:lem:proj:4} and \eqref{eq:lem:proj:5} into \eqref{eq:lem:proj:1}, we arrive at the following sufficient condition so that $\rank\left({\Pi_{\mS_n}[\X-\eta\widetilde{\nabla}]}\right) \leq r$, 
{\small
\begin{align*}
-r\mu_n -r\xi - \sqrt{r}\Vert{\X-\X^*}\Vert\left({\sqrt{r}\beta+\frac{2\sqrt{2}\Vert{\widetilde{\nabla}}\Vert}{\lambda_r(\X^*)}}\right)\geq   -r\left({\mu_n+\frac{\delta}{r+1}}\right) + r\left({\xi+ \beta\Vert{\X-\X^*}\Vert_F}\right),
\end{align*}}

which is equivalent to the condition
%\begin{eqnarray*}
$\Vert{\X-\X^*}\Vert_F \leq \left({2r\beta + \frac{2\sqrt{2r}{\Vert{\widetilde{\nabla}}\Vert}}{\lambda_r(\X^*)} }\right)^{-1}\left({\frac{r}{r+1}\delta - 2r\xi}\right)$.
%\end{eqnarray*}
Simplifying the above expression gives the result.
\end{proof}

\begin{lemma}\label{lem:martingaleDistNew}
Fix $p\in(0,1)$. Let $\X_1,\dots,\X_{T}$ be a sequence generated by Algorithm \ref{alg:sgd} such that for all $t\in[T-1]$, $\eta_t = \eta$ for some $\eta>0$ satisfying $1/\eta = \tilde{\Theta}(\sqrt{T})$, and with mini-batch size $L = O(\textrm{poly}(\log{}T))$.  Then, for any $\X^*\in\mX^*$ and any $T$ large enough, it holds with probability at least $1-p$ that for all $t\in[T]$:
\begin{align*}
\Vert{\X_t-\X^*}\Vert_F^2 &\leq \Vert{\X_1-\X^*}\Vert_F^2 + G^2T\eta^2 + \sqrt{40T\eta^2\sigma^2/L}\sqrt{\log(T/p)}.
\end{align*}
\end{lemma}
\begin{proof}
Define the auxiliary sequence $\{\Y_t\}_{t=1}^T$ as follows: $\Y_1=\X_1$ and for all $t\in[T-1]$, $\Y_{t+1} := \X_t - \eta_t\widehat{\nabla}_t$. Recall that with these definitions we have that for all $t\in[T]$, $\X_t = \Pi_{\mS_n}[\Y_t]$. 

Throughout the proof let us fix some optimal solution $\X^*\in\mX^*$. We begin with the observation that for all $t\in[T-1]$ it holds that
{\small
\begin{align*}
\E[\Vert{\Y_{t+1}-\X^*}\Vert_F^2~|~\X_t] &= \E[\Vert{\X_t-\eta_t\widehat{\nabla}_t-\X^*}\Vert_F^2~|~\X_t] \\
&\leq  \Vert{\X_t-\X^*}\Vert_F^2 - 2\eta_t\E[(\X_t-\X^*)\bullet\widehat{\nabla}_t~|~\X_t] + \eta_t^2G^2 \\
&=  \Vert{\X_t-\X^*}\Vert_F^2 - 2\eta_t(\X_t-\X^*)\bullet\nabla{}f(\X_t) + \eta_t^2G^2 \\
& \underset{(a)}{\leq} \Vert{\X_t-\X^*}\Vert_F^2 - 2\eta_t(f(\X_t)-f(\X^*)) + \eta_t^2G^2 \underset{(b)}{\leq}  \Vert{\Y_t-\X^*}\Vert_F^2 + \eta_t^2{}G^2,
\end{align*}}
where (a) follows from the convexity of $f(\cdot)$, and (b) follows since $\X_t$ is the projection of $\Y_t$ onto $\mS_n$.

%Similarly, it holds that
%\begin{align*}
%\Vert{\X_{t+1} - \X^*}\Vert_F^2 &\leq \Vert{\X_t-\X^*}\Vert_F^2 - 2\eta_t(\X_t-\X^*)\bullet\nabla{}f(\X_t) \\
%&+ 2\eta_t\Vert{\X_t-\X^*}\Vert_F\Vert{\nabla{}f(\X_t) - \widehat{\nabla}_t}\Vert_F + \eta_t^2G^2 \\
%&\leq \Vert{\X_t-\X^*}\Vert_F^2+2\sqrt{2}\eta_t\Vert{\nabla{}f(\X_t)-\widehat{\nabla}_t}\Vert_F + \eta_t^2G^2.
%\end{align*}

For all $t\in[T]$ define the random variable $Z_t := \Vert{\Y_t-\X^*}\Vert_F^2 - G^2\sum_{i=1}^{t-1}\eta_i^2$. Note that $Z_1,\dots,Z_T$ forms a submartingale sequence w.r.t. the filtration  $\mathbf{F}:=\{\mF_{t}:=\{\X_1,\dots,\X_{t}\}\}_{t=1}^{T-1}$. This holds since using %Lemma \ref{lem:distBound}, we have that for all $t\in[T-1]$:
the above inequality, we have that for all $t\in[T-1]$:
{\small
\begin{align*}
\E[Z_{t+1}|\mF_{t}] &= \E[\Vert{\Y_{t+1}-\X^*}\Vert_F^2|\X_t] - G^2\sum_{i=1}^{t}\eta_i^2 \leq  \Vert{\Y_{t}-\X^*}\Vert_F^2 + G^2\eta_t ^2- G^2\sum_{i=1}^{t}\eta_i^2 =Z_t.
%\E[Z_{t+1}|Z_1,\dots,Z_t] &= \E[\Vert{\Y_{t+1}-\X^*}\Vert_F^2~|~\Y_t] - G^2\sum_{i=1}^{t}\eta_i^2 \leq  \Vert{\Y_{t}-\X^*}\Vert_F^2 + G^2\eta_t ^2- G^2\sum_{i=1}^{t}\eta_i^2 =Z_t.
\end{align*}
}
We continue to show that this submartingale has bounded-differences and to upper-bound its variance. It holds for all $2\leq t\leq T$ that
%{\small
%\begin{align*}
%Z_t - \E[Z_t|Z_1,\dots,Z_{t-1}]  &= \Vert{\Y_t-\X^*}\Vert_F^2 - \E[\Vert{\Y_{t}-\X^*}\Vert_F^2|\X_{t-1}] \\
%&= \Vert{\X_{t-1}-\X^*-\eta_{t-1}\widehat{\nabla}_{t-1}}\Vert_F^2 - \E[\Vert{\X_{t-1}-\X^*-\eta_{t-1}\widehat{\nabla}_{t-1}}\Vert_F^2|\X_{t-1}] \\
%&\leq -2\eta_{t-1}(\X_{t-1}-\X^*)\bullet\widehat{\nabla}_{t-1} \\
%&+ \eta_{t-1}^2\Vert{\widehat{\nabla}_{t-1}}\Vert_F^2 + 2\eta_{t-1}(\X_{t-1}-\X^*)\bullet\E[\widehat{\nabla}_{t-1}|\X_{t-1}] \\
%&= 2\eta_{t-1}(\X_{t-1}-\X^*)\bullet\E[\widehat{\nabla}_{t-1}-\nabla{}f(\X_{t-1})|\X_{t-1}]+ \eta_{t-1}^2\Vert{\widehat{\nabla}_{t-1}}\Vert_F^2 \\
%&\leq  2\eta_{t-1}\Vert{\X_{t-1}-\X^*}\Vert_F\E[\Vert{\widehat{\nabla}_{t-1}-\nabla{}f(\X_{t-1})}\Vert_F|\X_{t-1}]+ \eta_{t-1}^2\Vert{\widehat{\nabla}_{t-1}}\Vert_F^2 \\
%&\leq 2\sqrt{2}\eta_{t-1}\sqrt{\sigma^2/L} + \eta_{t-1}^2G^2.
%\end{align*}}
{\small
\begin{align*}
Z_t - \E[Z_t|\mF_{t-1}]  &= \Vert{\Y_t-\X^*}\Vert_F^2 - \E[\Vert{\Y_{t}-\X^*}\Vert_F^2|\X_{t-1}] \\
&= \Vert{\X_{t-1}-\X^*-\eta_{t-1}\widehat{\nabla}_{t-1}}\Vert_F^2 - \E[\Vert{\X_{t-1}-\X^*-\eta_{t-1}\widehat{\nabla}_{t-1}}\Vert_F^2|\X_{t-1}] \\
&\leq -2\eta_{t-1}(\X_{t-1}-\X^*)\bullet\widehat{\nabla}_{t-1} + \eta_{t-1}^2\Vert{\widehat{\nabla}_{t-1}}\Vert_F^2 + 2\eta_{t-1}(\X_{t-1}-\X^*)\bullet\E[\widehat{\nabla}_{t-1}|\X_{t-1}] \\
&\underset{(a)}{\leq} 4\sqrt{2}\eta_{t-1}G + \eta_{t-1}^2G^2 \underset{(b)}{\leq} 6\eta{}G,
\end{align*}}
where (a) follows the Cauchy-Schwarz inequality and plugging the Euclidean diameter of $\mS_n$ and the bound $G$ on the norm of the stochastic gradients, and (b) holds for any $T$ sufficiently large. We continue to upper-bound the conditional variance. For any $2\leq t \leq T$ we have that
{\small
\begin{align*}
%\var\left({Z_t|Z_1,\dots,Z_{t-1}}\right) &= \var\left({\Vert{\Y_t-\X^*}\Vert_F^2|Z_1,\dots,Z_{t-1}}\right) \\
\var\left({Z_t|\mF_{t-1}}\right) &= \var\left({\Vert{\Y_t-\X^*}\Vert_F^2|\mF_{t-1}}\right) \\
&= \var\left({\Vert{\X_{t-1}-\X^* - \eta_{t-1}\widehat{\nabla}_{t-1}}\Vert_F^2|\X_{t-1}}\right)\\
 &=\var\left({\eta_{t-1}^2\Vert{\widehat{\nabla}_{t-1}}\Vert_F^2 - 2\eta_{t-1}(\X_{t-1}-\X^*)\bullet\widehat{\nabla}_{t-1}|\X_{t-1}}\right) \\
&\underset{(a)}{\leq} 2\eta_{t-1}^4\var\left({\Vert{\widehat{\nabla}_{t-1}}\Vert_F^2|\X_{t-1}}\right) + 8\eta_{t-1}^2\var\left({(\X_{t-1}-\X^*)\bullet\widehat{\nabla}_t|\X_{t-1}}\right). 
\end{align*}}
Thus,
{\small
\begin{align*}
%\var\left({Z_t|Z_1,\dots,Z_{t-1}}\right)&\leq 2\eta_{t-1}^4\var\left({\Vert{\widehat{\nabla}_{t-1}}\Vert_F^2|\X_{t-1}}\right) \\
\var\left({Z_t|\mF_{t-1}}\right)&\leq 2\eta_{t-1}^4\var\left({\Vert{\widehat{\nabla}_{t-1}}\Vert_F^2|\X_{t-1}}\right) + 8\eta_{t-1}^2\var\left({(\X_{t-1}-\X^*)\bullet(\widehat{\nabla}_t-\nabla{}f(\X_{t-1}))|\X_{t-1}}\right) \\
&\leq 2\eta_{t-1}^4\E\left[{\Vert{\widehat{\nabla}_{t-1}}\Vert_F^4|\X_{t-1}}\right] + 8\eta_{t-1}^2\E\left[{\left({(\X_{t-1}-\X^*)\bullet(\widehat{\nabla}_{t-1}-\nabla{}f(\X_{t-1}))}\right)^2|\X_{t-1}}\right] \\
&\underset{(b)}{\leq}  8\eta_{t-1}^2\Vert{\X_{t-1}-\X^*}\Vert_F^2\E\left[{\Vert{\widehat{\nabla}_{t-1}-\nabla{}f(\X_{t-1})}\Vert_F^2|\X_{t-1}}\right] + 2\eta_{t-1}^4G^4  \\
&\underset{(c)}{\leq} 2\eta_{t-1}^4G^4 + 16\eta_{t-1}^2\sigma^2/L,
\end{align*}}
where (a) follows since $\var(X+Y) \leq 2(\var(X)+\var(Y))$, (b) follows from the Cauchy-Schwarz inequality, and (c) follows from plugging the Euclidean diameter of $\mS_n$ and the variance of the mini-batch stochastic gradient.

Now, using a standard concentration argument for submartingales (see Theorem 7.3 in \cite{Chung06}, which we apply with parameters $a_i=0, \phi_i=0$), we have that for any $\Delta = O(\textrm{poly}(\log{}T))$, $L = O(\textrm{poly}(\log{}T))$, $1/\eta = \tilde{\Theta}(\sqrt{T})$, and $T$ large enough,
{\small
\begin{align*}
\Pr\left({Z_t \geq Z_1 + \Delta }\right) &\leq \exp\left(\frac{-\Delta^2}{\sum_{i=1}^{t-1}\left({4\eta_i^4G^4+32\eta_i^2\sigma^2/L}\right) + 4\eta{}G\Delta}\right) \\
&\leq \exp\left(\frac{-\Delta^2}{4T\eta^4G^4+32T\eta^2\sigma^2/L + 4\eta{}G\Delta}\right) \leq \exp\left(\frac{-\Delta^2}{40T\eta^2\sigma^2/L}\right).
\end{align*}}

%Now, using a standard concentration argument for submartingales (see Theorem 7.3 in \cite{Chung06}, which we apply with parameters $M=0, \phi_i=0$), we have that for any $\Delta >0$
%{\small
%\begin{align*}
%\Pr\left({Z_t \geq Z_1 + \Delta }\right) &\leq \exp\left(\frac{-\Delta^2}{2\sum_{i=1}^{t-1}\left({\left({2\eta_i^4G^4+16\eta_i^2\sigma^2/L}\right) + \left({2\sqrt{2}\eta_i^2\sigma^2/L + \eta_i^2G^2}\right)^2}\right)}\right) \\
%&\leq \exp\left({\frac{-\Delta^2}{\sum_{i=1}^{t-1}\left({8\eta_i^4G^4+64\eta_i^2\sigma^2/L}\right)}}\right).
%\end{align*}}

%Thus, for  $\Delta_t:=\sqrt{\sum_{i=1}^{t-1}\left({8\eta_i^4G^4+64\eta_i^2\sigma^2/L}\right)}\sqrt{\log\frac{1}{p'}}$
Thus, for  $\Delta:=\sqrt{40T\eta^2\sigma^2/L}\sqrt{\log\frac{1}{p'}}$ it holds that with probability at least $1-p'$ that
%\begin{align*}
$\Vert{\Y_t-\X^*}\Vert_F^2 \leq \Vert{\Y_1-\X^*}\Vert_F^2 +  G^2\sum_{i=1}^{t-1}\eta_i^2 + \Delta \leq \Vert{\Y_1-\X^*}\Vert_F^2 +  G^2T\eta^2 + \Delta$. 
%\end{align*}

Now, recalling that $\Y_1 = \X_1$, and since $\X_t$ is the projection of $\Y_t$ onto $\mS_n$, we have that with probability at least $1-p'$ it holds that
$\Vert{\X_t-\X^*}\Vert_F^2 \leq \Vert{\X_1-\X^*}\Vert_F^2 +  G^2T\eta^2+\Delta$. The Lemma now follows from setting $p' =p/T$ and using the union-bound for all $t\in[T]$.
\end{proof}

We can now prove Theorem \ref{thm:main}.

\begin{proof}[Proof of Theorem \ref{thm:main}]

Suppose for now that the iterates $\X_1,\dots,\X_T$ are computed using exact Euclidean projection. Then, standard results (see for instance proof of Theorem 6.1 in \cite{Bubeck15})  give that for any step-size $\eta >0$ it holds that
{\small
\begin{align*}
\E\Big[{\frac{1}{T}\sum_{t=1}^Tf(\X_t)}\Big] - f^* = O\left({\frac{\Vert{\X_1-\X^*}\Vert_F^2}{\eta{}T} + \eta{}G^2}\right).
\end{align*}}

Thus, for both options of the returned solution $\bar{\X}$ in Algorithm \ref{alg:sgd} (using the convexity of $f(\cdot)$ for \textbf{option II}), we have that
\vspace{-7pt}
{\small
\begin{align*}
\E\left[{f(\bar{\X})}\right] - f^* = O\left({\frac{\Vert{\X_1-\X^*}\Vert_F^2}{\eta{}T} + \eta{}G^2}\right).
\end{align*}}

In particular, using Markov's inequality, we have that with probability at least $3/4$ it holds that
{\small
\begin{align*}
f(\bar{\X}) - f^* = O\left({\frac{\Vert{\X_1-\X^*}\Vert_F^2}{\eta{}T} + \eta{}G^2}\right).
\end{align*}}

Using a standard Matrix Hoeffding concentration argument (see for instance \cite{Tropp12}), we have that with probability at least $9/10$, under the batch-size listed in the theorem, it holds that
$\forall t\in[T-1]$: $ \Vert{\widehat{\nabla}_t - \nabla{}f(\X_t)}\Vert \leq \delta/(2r)$.

Also, using Lemma \ref{lem:martingaleDistNew}, we have for any $T$ sufficiently large that with probability at least $7/8$ it holds that for all $1 \leq t \leq T-1$
%{\small
%\begin{align*}
$\Vert{\X_t-\X^*}\Vert_F^2 \leq \Vert{\X_1-\X^*}\Vert_F^2 + G^2T\eta^2 + \sqrt{40\eta^2\sigma^2T/L}\sqrt{\log(8T)}$.
%\end{align*}}
%{\small
%\begin{align*}
%\Vert{\X_t-\X^*}\Vert_F^2 &\leq \Vert{\X_1-\X^*}\Vert_F^2 + G^2T\eta^2 + \sqrt{8\eta^4G^4T+64\eta^2\sigma^2T/L}\sqrt{\log(8T)}.
%\end{align*}}

Thus, we have that for $\eta = \frac{R_0}{10G\log(8T)\sqrt{T}}$ and $L\geq (\sigma/G)^2R_0^{-2}$, combining all of the above guarantees, we have that with probability at least $1/2$, all following three guarantees hold:
\vspace{-6pt}
\begin{eqnarray*}
&i)~f(\bar{\X}) - f^* = O\left({\frac{GR_0\log{T}}{\sqrt{T}}}\right), \quad ii)~\forall t\in[T-1]:\quad \Vert{\X_t-\X^*}\Vert_F \leq R_0, &\\
&iii)~\forall t\in[T-1]:\quad \Vert{\widehat{\nabla}_t-\nabla{}f(\X_t)}\Vert \leq \frac{\delta}{2r}. &
\end{eqnarray*}
Thus, by invoking Lemma \ref{lem:goodProj}, with the above probability, for all $t\in[T]$ it holds that $\rank(\X_t) \leq r$. In particular, using \textbf{option I} in Algorithm \ref{alg:sgd}, the returned solution $\bar{\X}$ is also of rank at most $r$.
\end{proof}

\section{Preliminary Empirical Evidence}

\begin{table*}[h!]\renewcommand{\arraystretch}{1.1}
{\footnotesize
\begin{center}
  \caption{Information on experiments.  Column $\rank(\X^*)$ is taken from \cite{Garber19}. Column ``SVD rank" records the SVD rank used to compute the projection on each iteration, and the column ``max rank" records the maximum rank of any of the iterates produced by the algorithm.
  }\label{table:exps}
  \begin{tabular}{| c | c | c | c | c | c | c | c |} \hline
  \multicolumn{2}{|c|}{setting} &\multicolumn{3}{|c|}{low rank SGD}& \multicolumn{3}{|c|}{high rank SGD} \\ \hline
  trace ($\tau$) & $\rank(\X^*)$ &  step-size & SVD rank & max rank & step-size & SVD rank & max rank \\ \hline
  3000 & 10 & 0.02 & 10 & 10 & $1/\sqrt{t}$ & 250 & 250 \\ \hline
  3500 & 41 & 0.007 & 41 & 41 & $1/\sqrt{t}$ & 250 & 250 \\ \hline
  4000 & 70 & 0.005 & 70 & 70 & $1/\sqrt{t}$ & 250 & 250 \\ \hline
  \end{tabular}
\end{center}
}
\vskip -0.2in
\end{table*}\renewcommand{\arraystretch}{1}

\begin{figure*}[h!]
    \centering
        \includegraphics[width=1.45in]{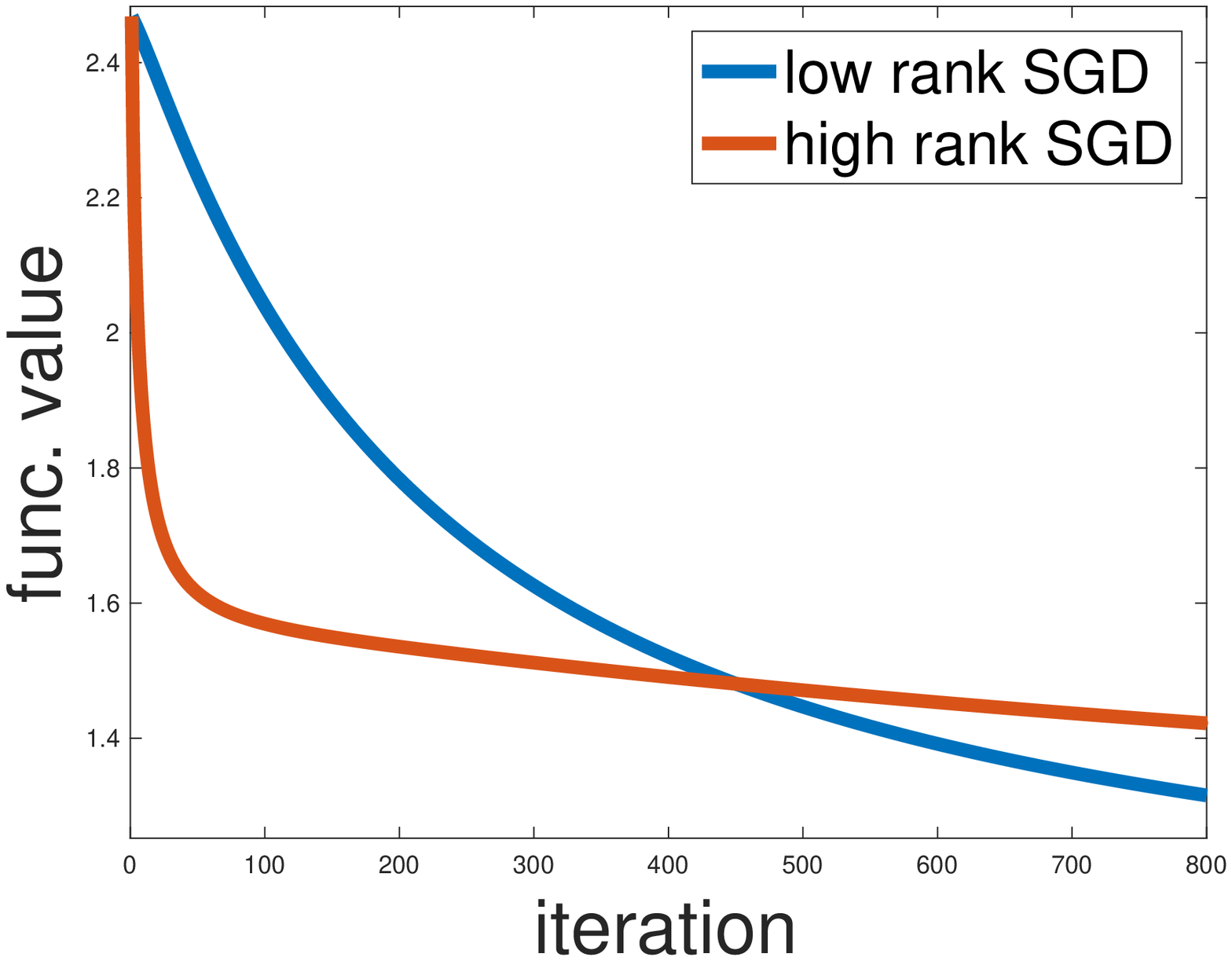}
        \includegraphics[width=1.45in]{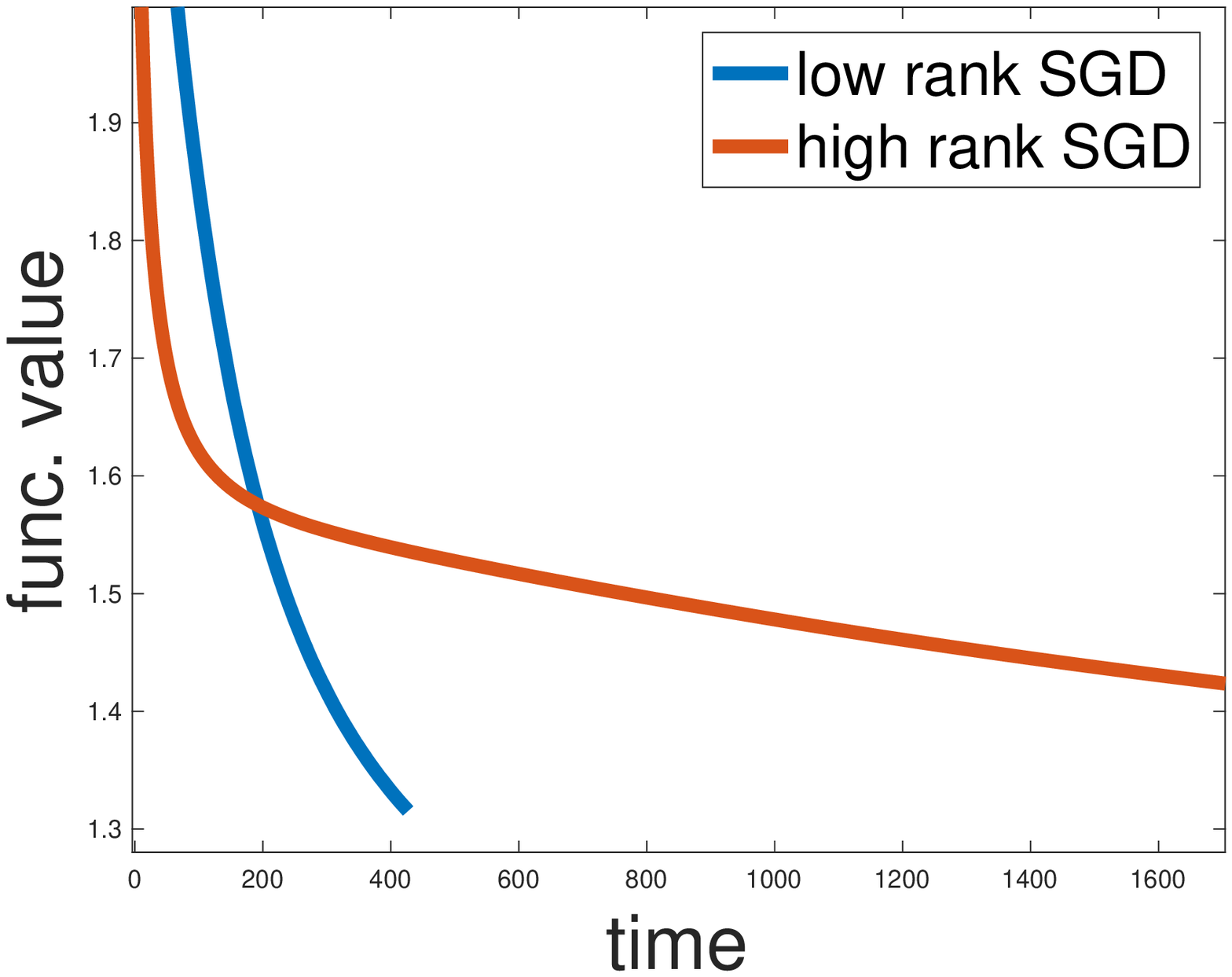}        
        \includegraphics[width=1.45in]{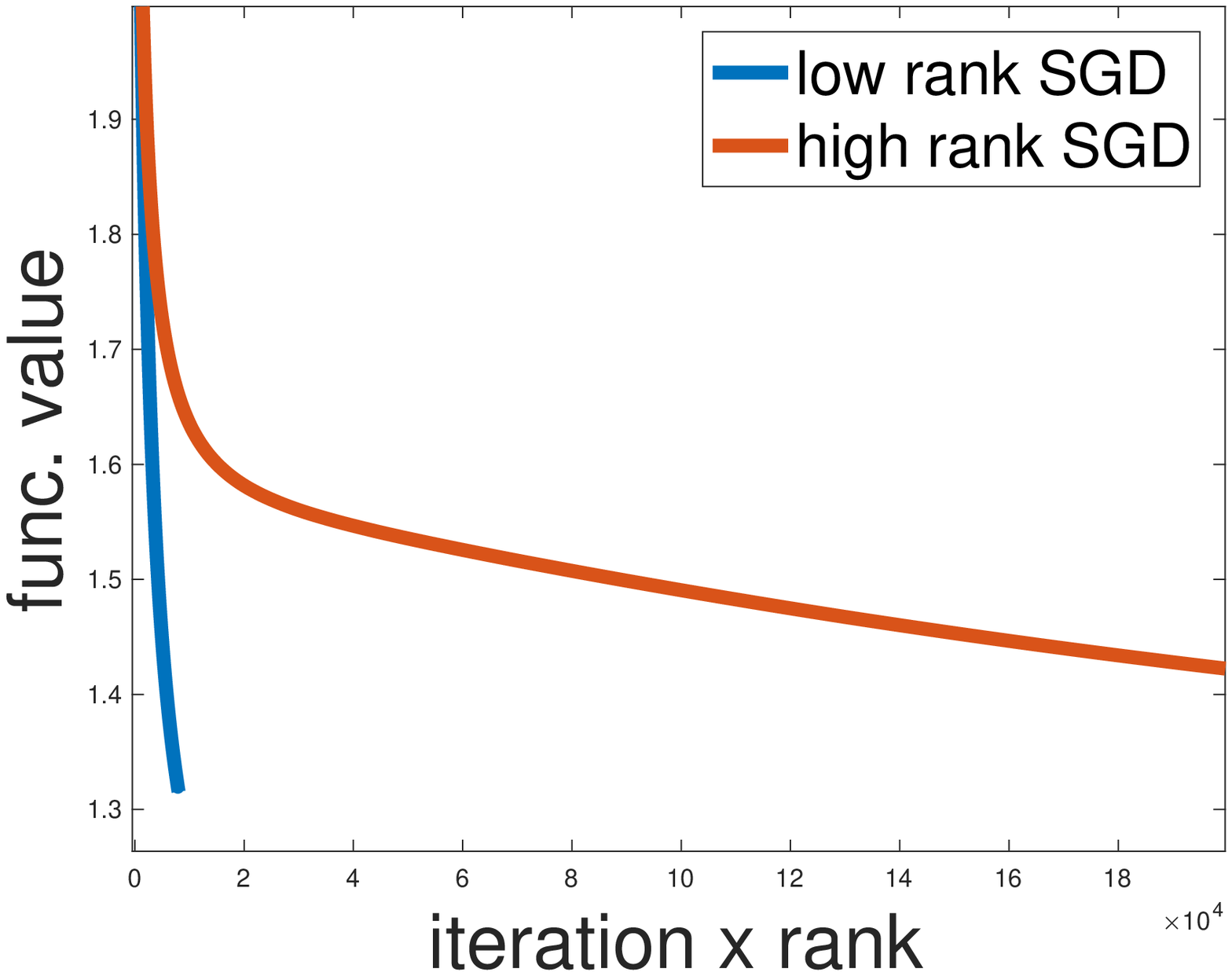}
	\caption*{trace = 3000}
        \includegraphics[width=1.45in]{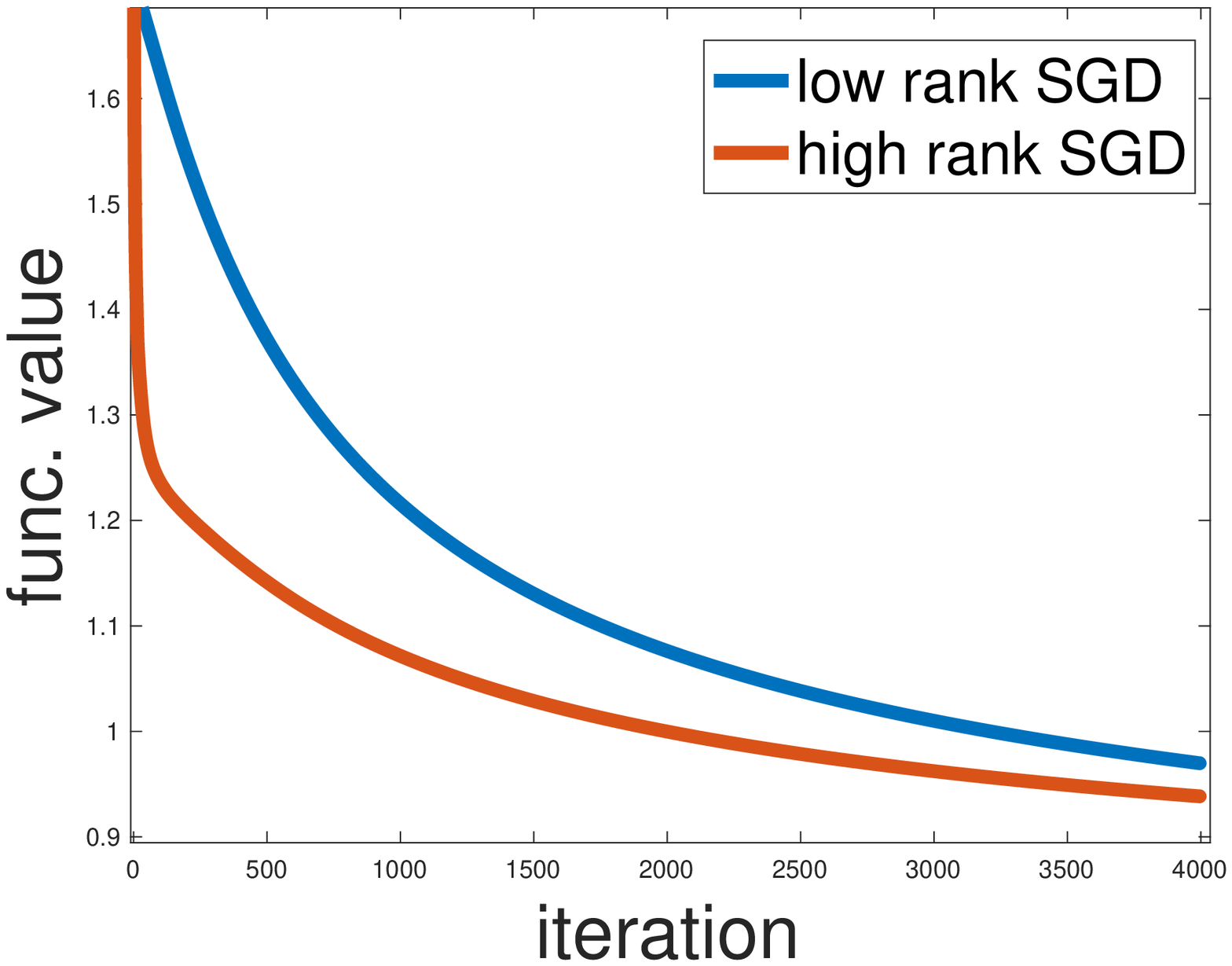}
        \includegraphics[width=1.45in]{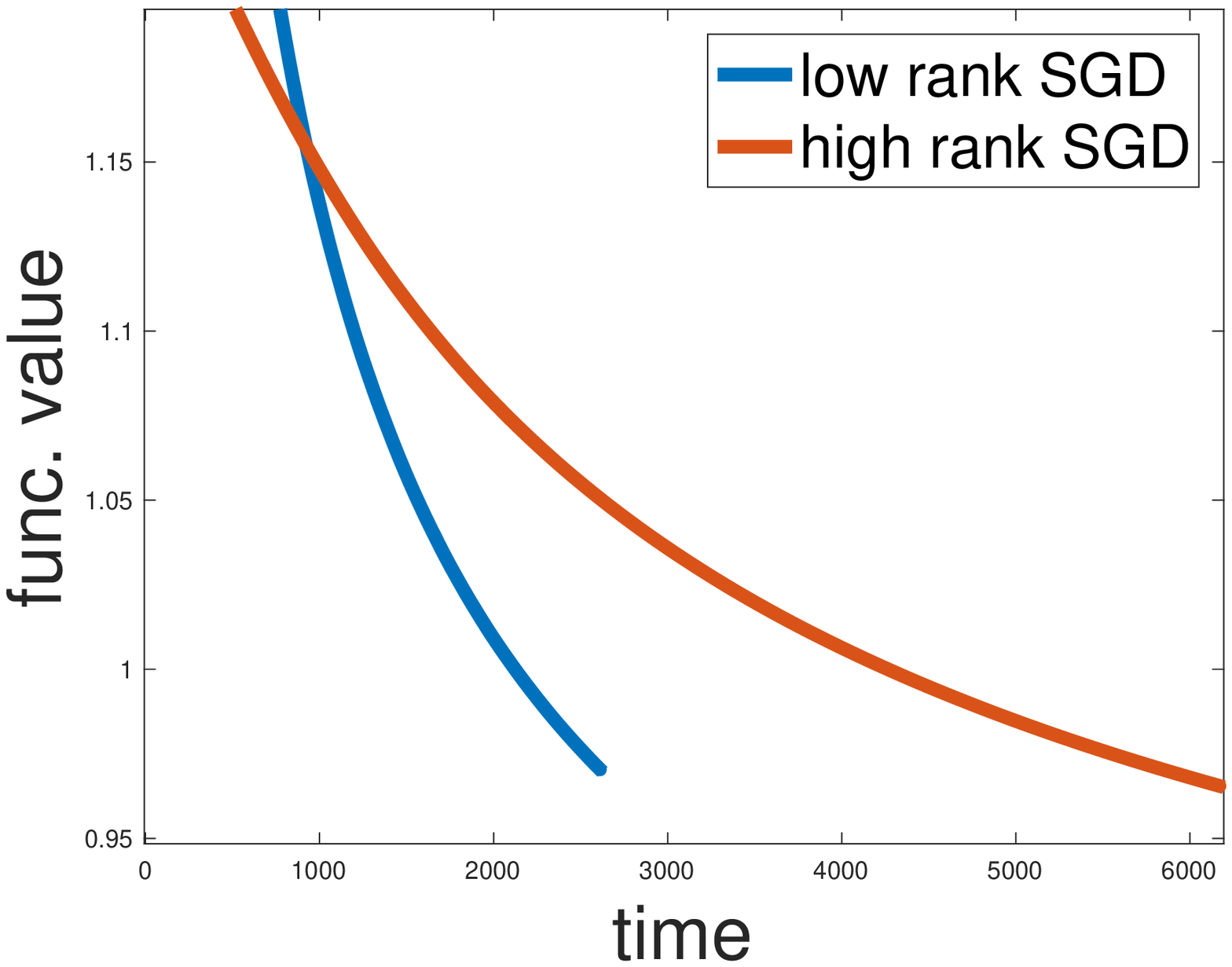}
        \includegraphics[width=1.45in]{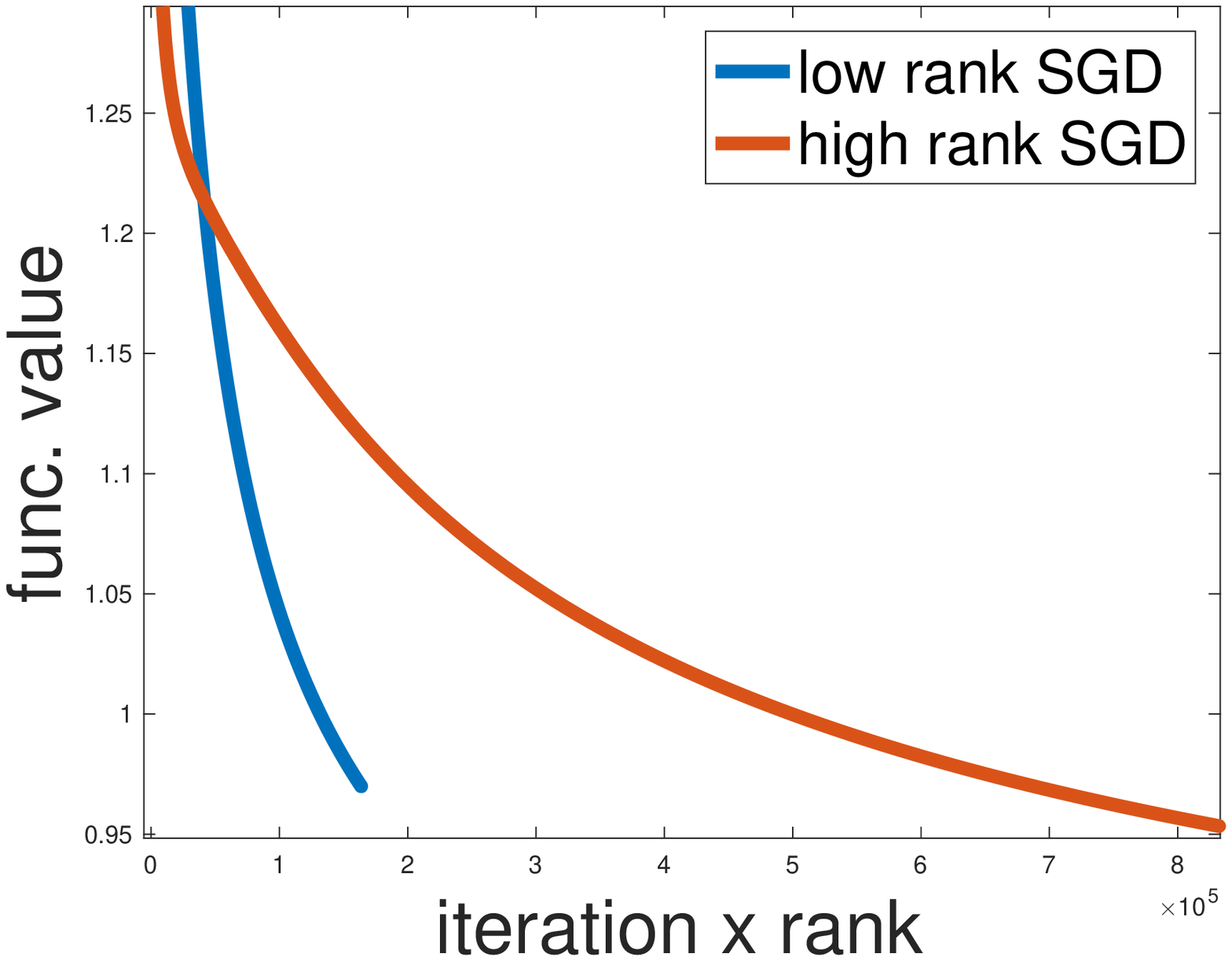}
        \caption*{trace = 3500}
        
        \includegraphics[width=1.45in]{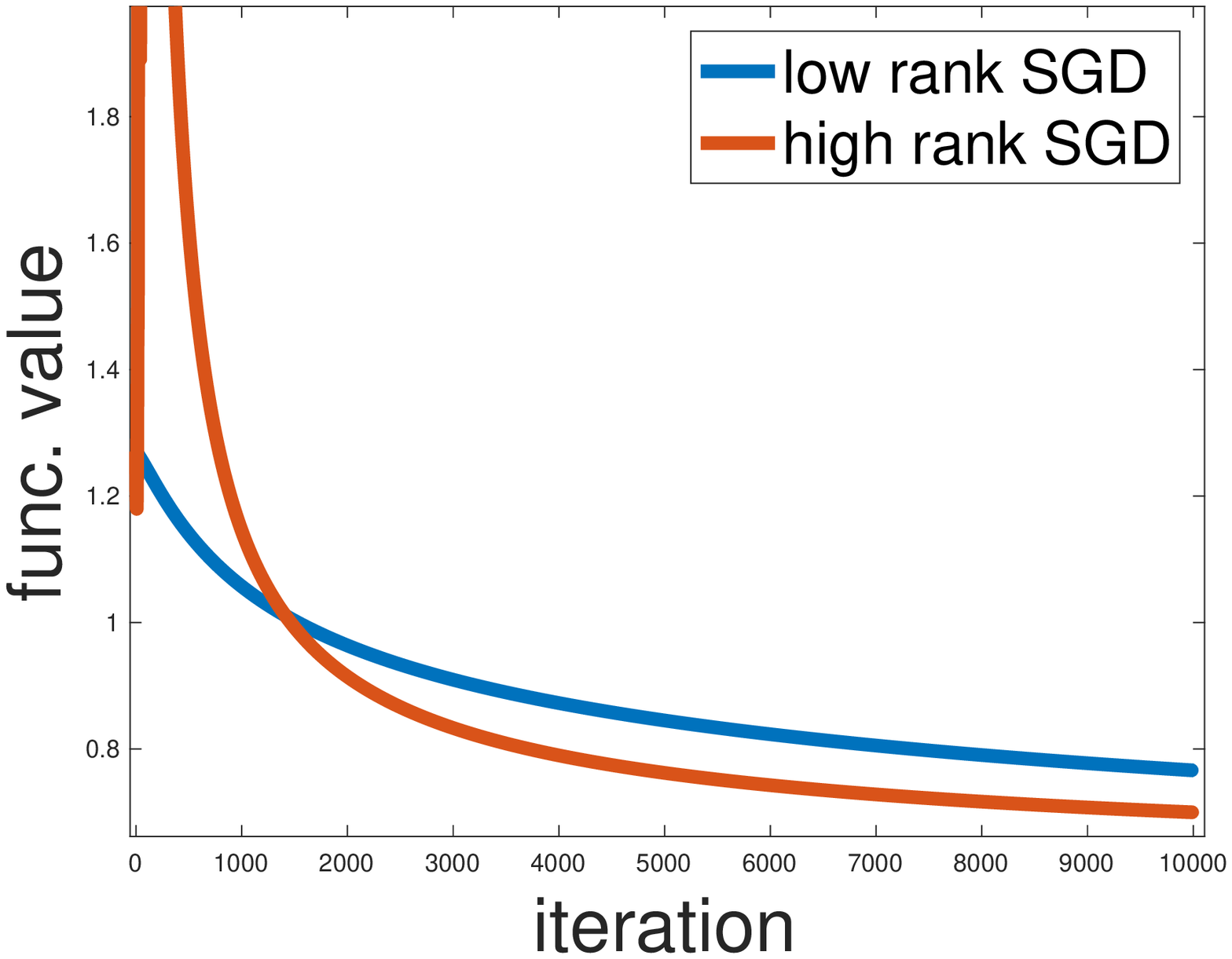}
        \includegraphics[width=1.45in]{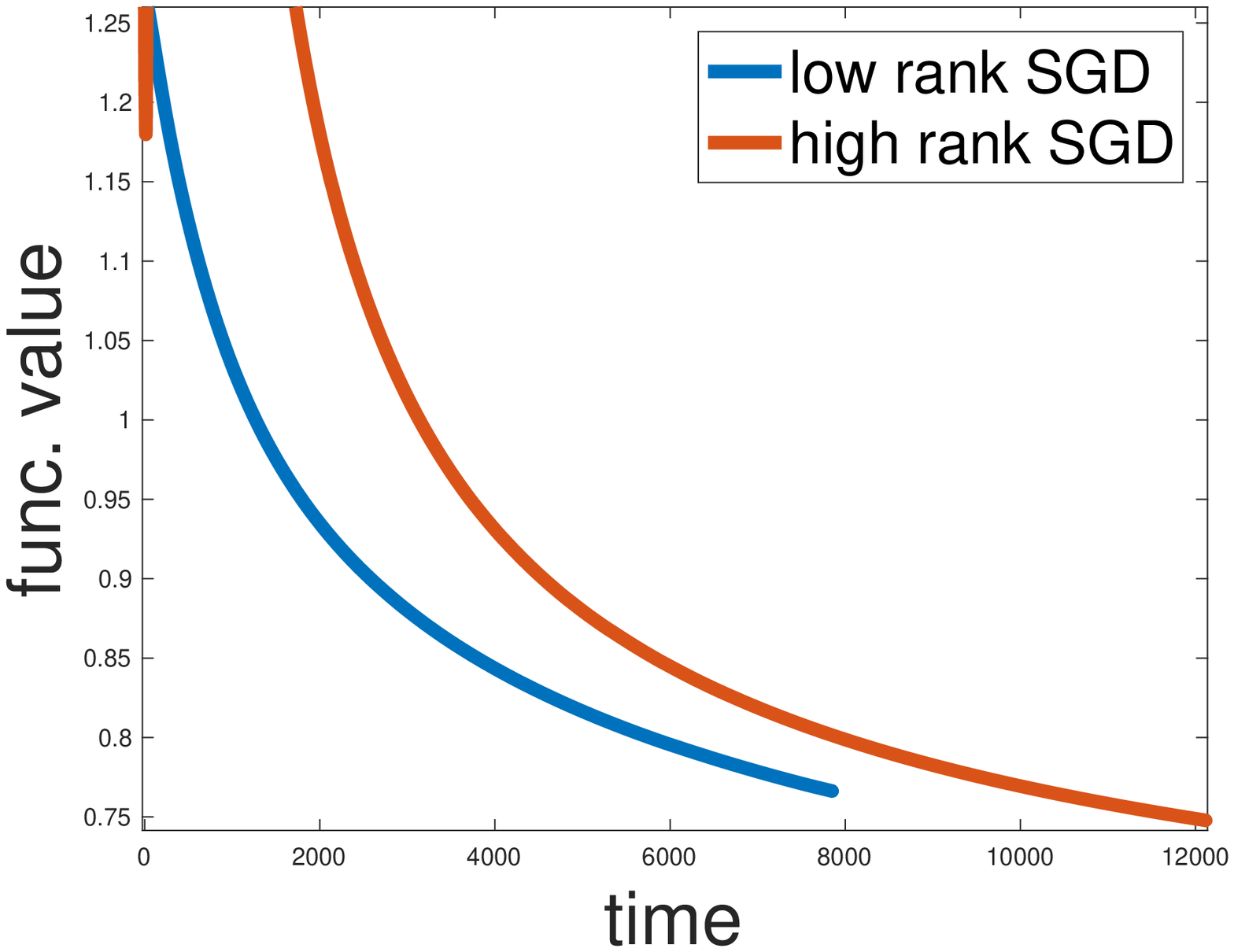}        
        \includegraphics[width=1.45in]{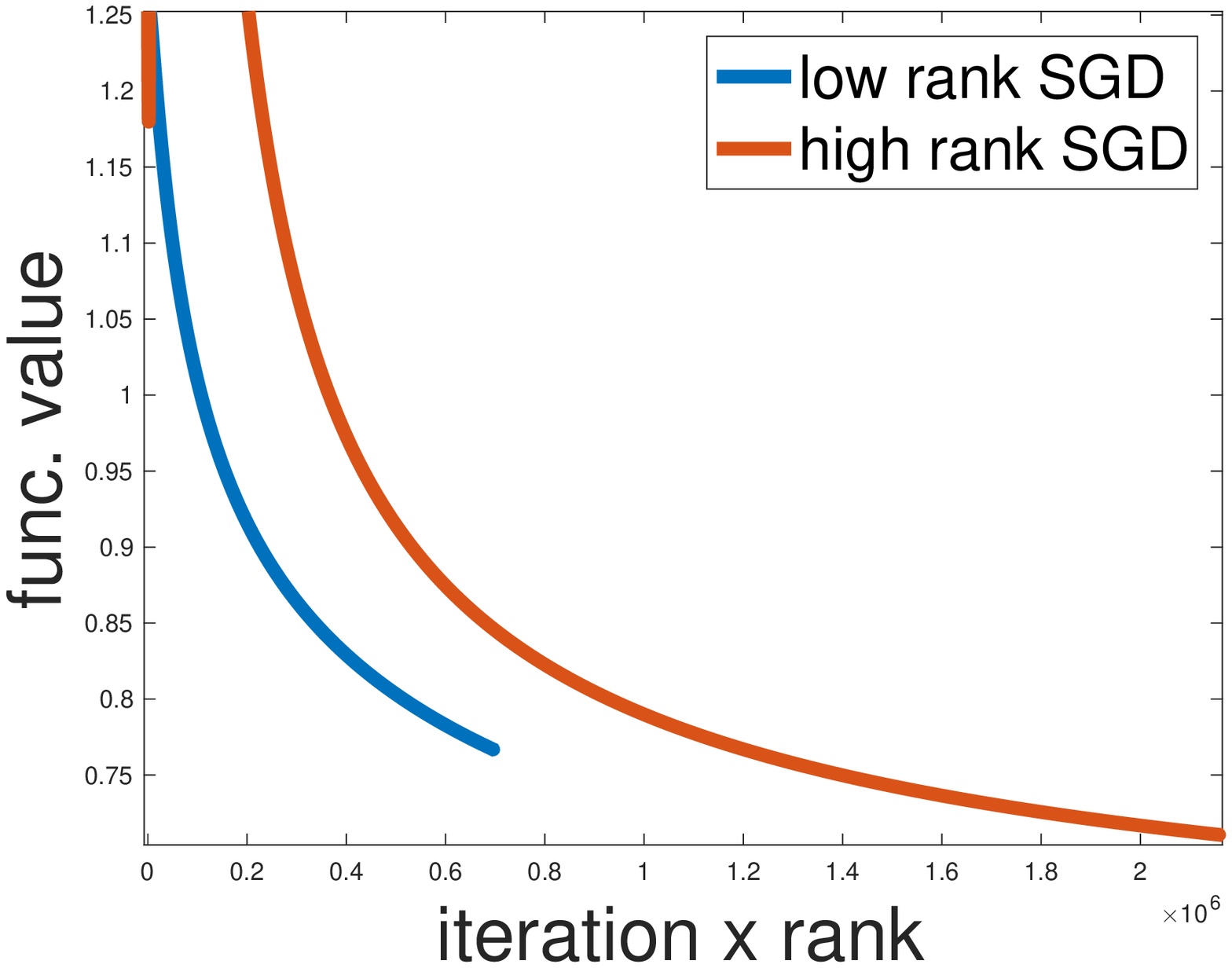}
        \caption*{trace = 4000}
    \caption{Performance of low rank SGD vs. (standard) high rank SGD. Each graph is the average of 5 i.i.d runs.}\label{fig:results}
\end{figure*}

The goal of this section is to motivate our theoretical investigation from an empirical point of view. Our main result, Theorem \ref{thm:main}, relies on an eigen-gap assumption (Assumption \ref{ass:gap}), a ``warm-start" initialization, and certain choice of step-size which depends on several parameters. In \cite{Garber19} it was already demonstrated that Assumption \ref{ass:gap} holds empirically for the highly popular matrix completion task. Here, we demonstrate  empirically, that SGD with low-rank projections converges correctly (i.e., the projection with low-rank SVD is always the accurate projection) for matrix completion with a very simple initialization scheme, and is competitive with a standard implementation of SGD, which uses high-rank SVD.

We use the standard MovieLens100K dataset (943x1682 matrix with 100,000 observed entries) \cite{Movielens16}. \footnote{We focus on this dataset and not larger ones because of the difficulty in scaling standard SGD, which requires high rank SVDs, to larger datasets.} Formally, our objective is the following:
\vspace{-8pt}
\begin{align}\label{eq:matrixComp}
\min_{\X\in\reals^{m\times n}: \Vert{\X}\Vert_*\leq \tau}\frac{1}{|S|}\sum_{(i,j,r)\in{}S}(\X_{i,j}-r)^2,
\end{align}
where $S$ is the set of observed entries (each entry is a triplet consisting of a matrix entry ($i,j$) and a scalar ranking ($r$)), and $\Vert\cdot\Vert_*$ is the trace norm of a matrix (sum of singular values). Problem \eqref{eq:matrixComp} could be directly formulated in the form of the canonical problem \eqref{eq:optProb} using standard manipulations (see for instance \cite{Jaggi10}). 

Following the experiments in \cite{Garber19}, we use different values for the trace norm bound $\tau$, which in turn affects the rank of the optimal solution $\X^*$. %In particular, in \cite{Garber19} (see Table 2) it was observed that for all values of $\tau$ used, the optimal solution indeed satisfies the eigen-gap assumption (Assumption \ref{ass:gap}). 
For both variants of SGD and for all experiments we use a batch-size of $L = 5000$ (5\% of the data).

For low rank SGD we always compute the projection using thin SVD with rank equal that of the optimal solution (see Table \ref{table:exps}). Also, we use a fixed step-size on all iterations which is tuned manually for every value of $\tau$, so that indeed throughout all iterations, the rank of the true projection is at most the rank of the SVD used (which we verify by examining the condition on the threshold parameter $\lambda$ in Lemma \ref{lem:spectrahedronProj}). Thus, to be clear, with this tuned step-size, the low rank projection is \textbf{always} (up to negligible numerical error) the correct projection, which matches our theoretical investigation.

For the standard (high rank) implementation of SGD, in order to allow for more realistic implementation, we set the SVD rank used to compute the projection to 250 (instead of $\min\{m,n\}$, see Table \ref{table:exps}). In all experiments we use a diminishing step-size of $\eta_t = 1/\sqrt{t}$ which follows the standard theoretical convergence results on SGD (up to constants, see \cite{Bubeck15} for instance), without additional tuning.

We initialize both variants with the same point (based on assigning each unobserved entry the mean value of the observed ones and taking a low rank SVD with rank that matches that of the optimal solution). Each experiment is the average of 5 i.i.d runs (due to the randomness in the mini-batch). The experiments were implemented in MATLAB with the \texttt{svds} command used to compute thin SVD.  We record the objective value \eqref{eq:matrixComp} as a function of the number of iterations (for both variants we calculate the objective at the average of iterates obtained so far), and the runtime (in seconds). Additionally, to give an approximate measure of time that is implementation-independent, we also plot the function value vs. the number of iterations scaled by the SVD rank used by each algorithm. This is because in theory (and also often in practice) the time to compute a thin SVD scales linearly with the rank of the SVD required.

It can be seen in Figure \ref{fig:results} that standard SGD (with step size $1/\sqrt{t}$) seems to exhibit faster converge rates in terms of \#iterations (perhaps with $\tau=3000$ being the exception), due to the smaller step-size required by the low rank variant to guarantee low rank projections. However, when examining either the runtime or the convergence rate scaled by SVD rank, we see that as expected, low rank SGD is significantly faster. Also, as recorded in Table \ref{table:exps}, while all iterates of low rank SGD indeed remain low rank, the iterates of high rank SGD always reach at some point the maximal rank used of 250, indicating that using a larger step-size indeed comes with a price.% of going through high rank matrices. 

\section{Discussion}
The main message we hope to convey in this work is that, perhaps in contrast to current popular belief, convex optimization methods can indeed be efficient for large-scale low-rank matrix problems, from the point of view of both theory and practice. We thus believe that it is worthwhile to continue studying their efficient implementations, perhaps under suitable assumptions.

There are two avenues for further research which could be of interest. First, Theorem \ref{thm:main} holds only with constant probability and not with high probability. Second, our analysis requires taking mini-batches. Since our objective is smooth, we may expect that these mini-batches will improve the convergence rate (see for instance Theorem 6.3 in \cite{Bubeck15} which, roughly speaking, shows the rate improves by a factor of $\sqrt{L}$, where $L$ is the mini-batch size). Unfortunately, our current analysis requires taking too small step-sizes (in order for the iterates to stay close enough to the optimal solution, see Lemma \ref{lem:martingaleDistNew}) to leverage the variance reduction due to the mini-batch.

\section*{Acknowledgments}
This research was supported by the ISRAEL SCIENCE FOUNDATION (grant No. 1108/18).

\bibliography{bib}

\appendix

\section{Proof of Lemma \ref{lem:robustRank}}

We first restate the lemma and then prove it.

\begin{lemma}
Let $f:\mbS^n\rightarrow\reals$ be $\beta$-smooth and convex. Let $\X^*\in\mS_n$ be an optimal solution of rank $r$ to the optimization problem $\min_{\X\in\mS_n}f(\X)$. Let $\mu_1,\dots,\mu_n$ denote the eigenvalues of $\nabla{}f(\X^*)$ in non-increasing order. Then, $\mu_{n-r}=\mu_{n}$ if and only if for any arbitrarily small $\zeta>0$ it holds that
\begin{eqnarray*}
\rank(\Pi_{(1+\zeta)\mS_n}[\X^*-\beta^{-1}\nabla{}f(\X^*)]) > r,
\end{eqnarray*}
where $(1+\zeta)\mS_n = \{(1+\zeta)\X~|~\X\in\mS_n\}$, and $\Pi_{(1+\zeta)\mS_n}[\cdot]$ denotes the Euclidean projection onto the convex set $(1+\zeta)\mS_n$.
\end{lemma}

\begin{proof}
Let us write the eigen-decomposition of $\X^*$ as $\X^*=\sum_{i=1}^r\lambda_i\v_i\v_i^{\top}$. It follows from the optimality of $\X^*$ that for all $i\in[r]$, $\v_i$ is also an eigenvector of $\nabla{}f(\X^*)$ which corresponds to the smallest eigenvalue $\mu_n$ (see Lemma 7 in \cite{Garber19}). Thus, if we let $\rho_1,\dots,\rho_n$ denote the eigenvalues (in non-increasing order) of $\Y := \X^*-\beta^{-1}\nabla{}f(\X^*)$, it holds that
\begin{eqnarray*}
\forall i\in[r]: \quad \rho_i &=& \lambda_i - \beta^{-1}\mu_n;\\
\forall i>r: \quad \rho_i &=& \lambda_i - \beta^{-1}\mu_{n-i+1}.
\end{eqnarray*}

Recall that $\sum_{i=1}^r\lambda_i =1$ and $\lambda_{r+1} = 0$.

It is well known that for any matrix $\M\in\mbS^n$ with eigen-decomposition $\M=\sum_{i=1}^n\sigma_i\u_i\u_i^{\top}$, the projection of $\M$ onto the set $(1+\zeta)\mS_n$, for any $\zeta \geq 0$ is given by
\begin{eqnarray*}
\Pi_{(1+\zeta)\mS_n}[\M] = \sum_{i=1}^n\max\{0,~\sigma_i-\sigma\}\u_i\u_i^{\top},
\end{eqnarray*}
where $\sigma\in\reals$ is the unique scalar such that $\sum_{i=1}^n\max\{0,~\sigma_i-\sigma\} = 1+\zeta$.

Now, we can see that $\rank(\Pi_{(1+\zeta)\mS_n}[\Y]) \leq r$ if and only if $\sigma \geq \rho_{r+1} = -\beta^{-1}\mu_{n-r}$. However, in this case, we have
\begin{align*}
1 + \zeta &= \sum_{i=1}^n\max\{0,~\rho_i-\sigma\}  = \sum_{i=1}^r\max\{0,~\rho_i-\sigma\} \leq \sum_{i=1}^r\max\{0,~\rho_i-(-\beta^{-1}\mu_{n-r})\} \\
&= \sum_{i=1}^r(\rho_i-(-\beta^{-1}\mu_{n-r})) = \sum_{i=1}^r(\lambda_i +\beta(\mu_{n-r}-\mu_n)) \\
&= 1 + \beta{}r(\mu_{n-r}-\mu_{n}) < 1+\zeta \quad \forall \zeta > \beta{}r(\mu_{n-r}-\mu_n).
\end{align*}

Thus, for any fixed $\zeta > 0$, it follows that $\rank(\Pi_{(1+\zeta)\mS_n}[\Y]) \leq r$ if and only if  $\beta{}r(\mu_{n-r}-\mu_n) \geq \zeta$. This proves the lemma.
%Thus, it follows that it must hold that $\sigma < -\beta^{-1}\mu$, in which case $\rank(\Pi_{(1+\zeta)\mS_n}[\Y]) > r$.
\end{proof}

\section{Proof of Corollary \ref{corr:samplecomplexity}}

We first restate the corollary.

\begin{corollary}
The overall sample complexity to achieve $f(\bar{\X})-f^*\leq \epsilon$ with probability at least $1/2$, when initializing from a ``warm-start'', is upper-bounded by $\tilde{O}\left({\epsilon^{-2}\max\{\sigma^2,\lambda_r^2(\X^*)rG^2\}}\right)$ (note that $\lambda_r(\X^*) \leq 1/r$).
\end{corollary}
\begin{proof}
The overall sample complexity is given simply by the number of iterations to reach $\epsilon$ error times the size of the minibatch and is thus upper-bounded by:
{\small
\begin{align*}
&\tilde{O}\left({\frac{G^2R_0^2}{\epsilon^2}\cdot{}\max\{(\sigma/G)^2R_0^{-2},\frac{B^2r^2}{\delta^2}\}}\right) = \tilde{O}\left({\frac{1}{\epsilon^2}\max\{\sigma^2,\frac{R_0^2G^2B^2r^2}{\delta^2}\}}\right)=\\
&\tilde{O}\left({\frac{1}{\epsilon^2}\max\{\sigma^2,\left({\frac{\lambda_r(\X^*)}{\lambda_r(\X^*)r\beta+\sqrt{r}B}}\right)^2G^2B^2r^2\}}\right)=\\
& \tilde{O}\left({\frac{1}{\epsilon^2}\max\{\sigma^2,\left({\frac{\lambda_r(\X^*)}{\sqrt{r}B}}\right)^2G^2B^2r^2\}}\right) = \tilde{O}\left({\frac{1}{\epsilon^2}\max\{\sigma^2,\lambda_r^2(\X^*)rG^2\}}\right).
\end{align*} 
}
\end{proof}

\end{document}